\documentclass{article}
\usepackage[accepted]{icml2025}

\usepackage{amssymb}
\usepackage{mathtools}
\usepackage{amsthm}

\newtheorem*{theorem*}{Theorem}
\newtheorem{theorem}{Theorem}
\newtheorem{definition}{Definition}

\usepackage[utf8]{inputenc} 
\usepackage[T1]{fontenc}    
\usepackage{hyperref}       
\usepackage{url}            
\usepackage{booktabs}       
\usepackage{amsfonts}       
\usepackage{nicefrac}       
\usepackage{microtype}      
\usepackage{xcolor}         
\usepackage{graphicx} 

\usepackage{amsmath}

\usepackage{comment}

\graphicspath{ {images/} }

\usepackage[nameinlink]{cleveref}
\usepackage{subcaption}

\usepackage{etoolbox}       
\newtoggle{show_comments}
\toggletrue{show_comments}

\definecolor{darkgreen}{RGB}{0,125,0}

\newtheorem{corollary}{Corollary}[theorem]

\definecolor{red}{HTML}{DB4437}
\definecolor{blue}{HTML}{4285F4}
\definecolor{green}{HTML}{0F9D58}
\usepackage[nameinlink]{cleveref}

\begin{document}

\twocolumn[
\icmltitle{Jackpot! Alignment as a Maximal Lottery}



\icmlsetsymbol{equal}{*}

\begin{icmlauthorlist}
\icmlauthor{Roberto-Rafael Maura-Rivero}{GDM,LSE}
\icmlauthor{Marc Lanctot}{GDM}
\icmlauthor{Francesco Visin}{GDM}
\icmlauthor{Kate Larson}{GDM,uwaterloo}
\end{icmlauthorlist}

\icmlaffiliation{LSE}{London School of Economics}
\icmlaffiliation{GDM}{Google DeepMind}
\icmlaffiliation{uwaterloo}{University of Waterloo}

\icmlcorrespondingauthor{Roberto-Rafael Maura-Rivero}{r.maura-rivero@lse.ac.uk}

\icmlkeywords{Social Choice Theory, Alignment, Nash Learning from Human Feedback, Reinforcement Learning from Human Feedback, RLHF}

\vskip 0.3in
]
\printAffiliationsOnly


\begin{abstract}

 Reinforcement Learning from Human Feedback (RLHF), the standard for aligning Large Language Models (LLMs) with human values, is known to fail to satisfy properties that are intuitively desirable, such as respecting the preferences of the majority \cite{ge2024axioms}. To overcome these issues, we propose the use of a probabilistic Social Choice rule called \emph{maximal lotteries} as a replacement for RLHF. We show that a family of alignment techniques, namely Nash Learning from Human Feedback (NLHF) \cite{munos2023nash} and variants, approximate maximal lottery outcomes and thus inherit its beneficial properties.
 We confirm experimentally that our proposed methodology handles situations that arise when working with preferences more robustly than standard RLHF, including supporting the preferences of the majority, providing principled ways of handling non-transitivities in the preference data, and robustness to irrelevant alternatives. 
 This results in systems that better incorporate human values and respect human intentions.

\end{abstract}

\section{Introduction}\label{sec:introduction}

\begin{figure}[ht]
    \centering
    \includegraphics[width=0.8\linewidth]{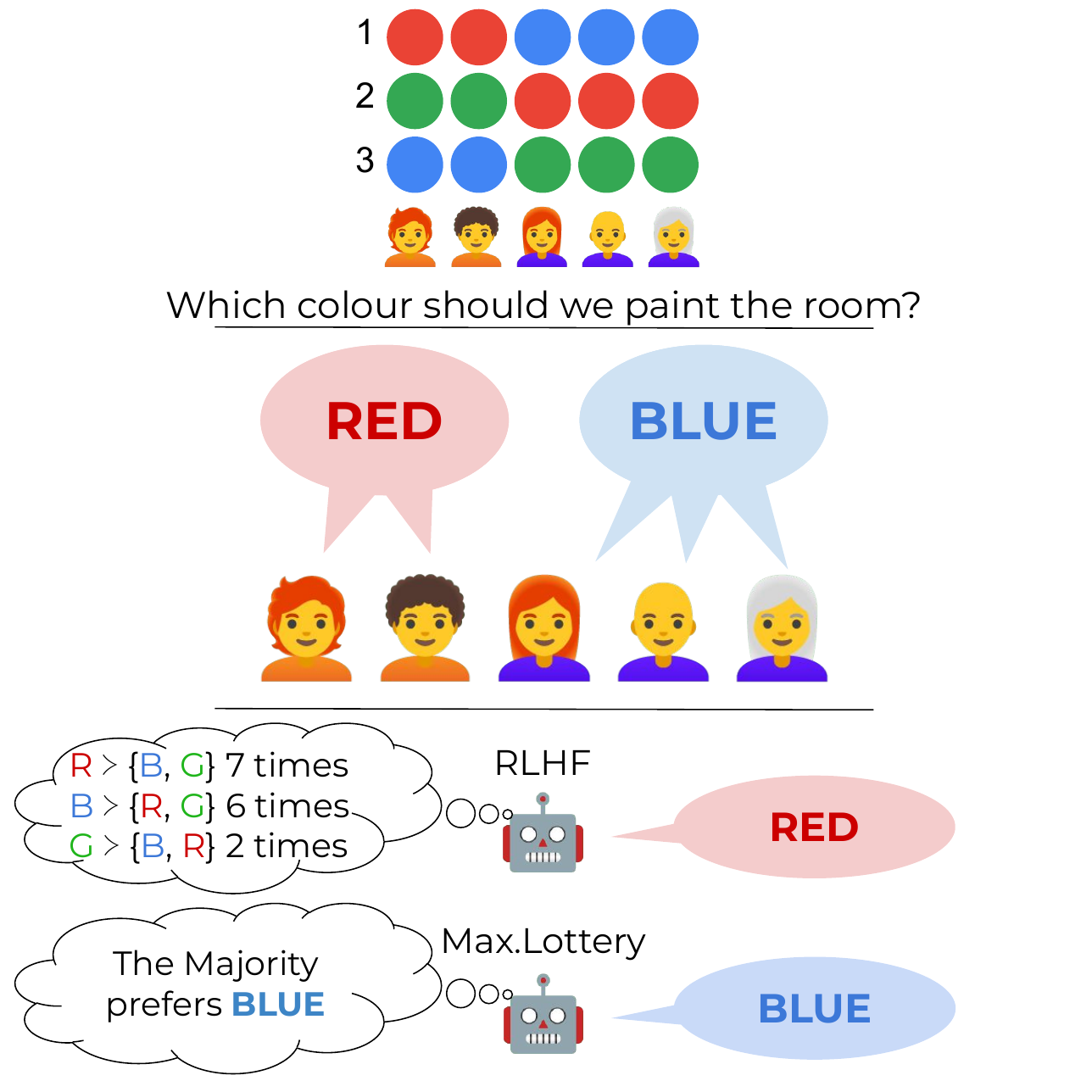}
    \caption{Although \textbf{\textcolor{blue}{B} is the option preferred by the majority}, LLMs aligned with RLHF fail to capture that, returning \textbf{\textcolor{red}{R}}. Thus, RLHF violates major democratic properties such as majority rule, while methods that emulate Maximal Lotteries satisfy them.}
    \label{fig:example-image}
\end{figure}

Reinforcement Learning from Human Feedback (RLHF) has emerged as the de-facto standard to align Large Language Models (LLMs) with human values and preferences. 
Using ideas from revealed preference theory in economics, current RLHF methods adapt the LLM’s distribution of generated text or tokens so as to maximize a reward model learned from the ratings of human evaluators. 

Despite its widespread use in fine tuning LLMs~\cite{touvron2023llamaopenefficientfoundation,openAI,anthropic,bard}, it has been recognized that current approaches suffer from fundamental limitations in the human feedback, the reward model, and training the policy~\cite{casper2023open}. These challenges  include tradeoffs between the richness and efficiency of feedback types, with binary preferences between pairs of examples being more prominent~\cite{christiano2017deep}, the assumption that a single reward function can represent a diverse population, which leads to current approaches modelling differences among evaluators as noise instead of important sources of disagreement~\cite{baumler-etal-2023-examples} or ambiguity~\cite{schaekermann2018resolvable},  and reward models failing to generalize even with perfect training data~\cite{skalse2023misspecificationinversereinforcementlearning}.

A number of recent papers have begun to explore alignment problems through the lens of Social Choice~\cite{ge2024axioms, dai2024mapping, mishra2023ai,siththaranjan2024distributional}, which provides principled methods for aggregating preferences, particularly for diverse populations, as well as tools and insights to understand the benefits and challenges that arise from that~\cite{brandt2016handbook}. In a recent position paper, \citeauthor{conitzerposition} argued that methods from Social Choice Theory provide alternative approaches to current RLHF methodologies~\cite{conitzerposition}.

In this paper we make the argument that a Probabilistic Social Choice function, \emph{maximal lotteries}~\cite{fishburn1984probabilistic}, is particularly well suited for RLHF and alignment problems. In particular,
\begin{itemize}
    \item We propose an alternative alignment method to RLHF based on Maximal Lotteries, a stochastic voting rule from Social Choice Theory.
    \item We formally prove that game-theoretic approaches to preference modeling in RLHF, specifically \textit{Nash Learning from Human Feedback}~\cite{munos2023nash} and its variants~\cite{calandriello2024human, swamy2024minimaximalist}, emulate maximal lotteries. 
    \item Through controlled experiments, we show that our approach produces LLM outputs that better reflect aggregate human preferences compared to standard RLHF, supporting the preferences of the majority, providing principled ways of handling non-transitivities in the preferences, and robustness to of irrelevant alternatives.
\end{itemize}

\section{Background}\label{sec:preliminaties_and_notation}

\subsection{Reinforcement Learning from Human Feedback}\label{subsec:RLHF}

RLHF involves training a reward model and then using this model to guide a policy (the LLM) through reinforcement learning. A reward model $r_\theta(x, y)$ is trained to predict a score indicating how good the response $y$ is to the prompt $x$. This model is learned from a dataset of pairwise comparisons, where human annotators indicate their preferences. The training objective involves maximizing the likelihood of correctly predicting the preferred option using a binary cross-entropy loss, $\mathcal{L}(\theta)=-E_{ (x,y^+,y^-)}[\log(\sigma(r_\theta(x,y^+)-r_\theta(x,y^-)))$. Here,  $(x, y^+, y^-)$ represents a data point, with $x$ being the prompt, $y^+$ the preferred completion (the ``winner''), and $y^-$ the less preferred completion (the ``loser'') , $r_\theta$ is the reward model parameterized by $\theta$ and $\sigma$ is the sigmoid function.
This loss function is based on the Bradley-Terry model~\cite{Rafailov23DPO} which is the foundation of the classical Elo rating system~\cite{Elo78}.
While this model is widely-used for RLHF, it has several well-documented problems that could affect preference learning~\cite{Nihar17Simple,Balduzzi19,bertrand2023on,lanctot2023evaluating,munos2023nash}.

The LLM, acting as the policy $\pi_\phi$ parameterized by $\phi$, is then trained using reinforcement learning algorithms like Proximal Policy Optimization (PPO)~\cite{schulman2017proximal}.  A simplified objective (ignoring regularization) can be written as $\max_\phi \mathbb{E}_{x \sim \mathcal{D}, y \sim \pi_\phi(x)} [r_\theta(x, y)]$,
  where $\mathcal{D}$ is the distribution of prompts.
  This loss encourages the LLM to generate completions that receive high reward.

\subsection{Social Choice Theory}\label{subsec:social_choice}

The central problem addressed by Social Choice Theory is how to aggregate the preferences of a population so as to reach some optimal collective decision.
Assume there is a population $\mathcal{P}$ of individuals, each of whom have preferences over some set of options $\mathcal{Y}$. Given a pair of different alternatives $a, b\in\mathcal{Y}$, an individual $i\in\mathcal{P}$ is able to report that either they prefer $a$ to $b$ ($a\succ_i b$) or $b$ to $a$ ($b\succ_i a$).\footnote{For ease of exposition we will assume strict preferences in the rest of the paper, but results can be extended to weak preferences.} A \textbf{Social Choice function} $f$ is a map that assigns to each preference profile $\{\succ_i\}_{i\in \mathcal{P}}$ a winning alternative in $\mathcal{Y}$, i.e. $f(\{\succ_i\}_{i\in \mathcal{P}})\in \mathcal{Y}$. A \textbf{Probabilistic Social Choice function} $\rho$ is a similar concept that returns a distribution over the set of alternatives, $\rho(\{\succ_i\}_{i\in \mathcal{P}})\in\Delta(\mathcal{Y})$.
 
Much of Social Choice Theory is axiomatic in nature~\cite{brandt2016handbook}, in that the field tries to understand what properties Social Choice functions can and should exhibit. For example, a \textbf{Condorcet winner} defines a fairly intuitive concept: an alternative $a$ is a Condorcet winner if $a$ preferred by more individuals than $b$ in every head-to-head pairing for every $b\in\mathcal{Y}$ (for a more formal definition of Condorcet winners, see~\Cref{def:condorcet_winner} in \Cref{sec:properties}). Social Choice functions that are guaranteed to return a Condorcet winner when it exists are called \emph{Condorcet-consistent} rules. Not all Social Choice functions are Condorcet consistent, like the well known class of \textbf{scoring rules} which include plurality and Borda. These rules translate individual's preference rankings over $m$ alternatives to a score vector $\mathbf{w}=(w_1,\ldots, w_m)$ where $w_1\geq w_2\geq\ldots w_m$ and $w_1>w_m$. Each alternative's total score is obtained by summing the individual scores assigned by all voters. Scoring rules can be interpreted as Social Choice functions where alternatives are simply sorted according to their scores and the top option is returned.

The Borda rule, for example, uses a scoring vector $\mathbf{w}=(m-1,m-2,\ldots, 0)$.\footnote{Since Borda is a C2 rule according to Fishburn's classification, it can be computed by using pairwise comparisons. The details are beyond the scope of this paper but we refer an interested reader to ~\cite{brandt2016handbook}.}  While the scoring rules are not guaranteed to return Condorcet winners, they exhibit other desirable properties.  Selecting a Social Choice function always implies a tradeoff in properties it will support, as crystalized by Arrow's Impossibility Theorem~\cite{arrow1950difficulty}, so clear specifications as to what properties are important in the context of an application of Social Choice is of critical importance.

\section{Alignment as a Social Choice Problem}

We support the view of several works in the literature~\cite{ge2024axioms, dai2024mapping, mishra2023ai, conitzerposition} that the alignment problem may be formalized as a Social Choice problem. 
Under this lens, given a prompt $x$, the set of all possible responses (up to a finite maximum length $L$) forms the set of alternatives $\mathcal{Y}$ the LLM has to choose from. The population $\mathcal{P}$ is then the set of individuals that report their preferences over $\mathcal{Y}$ in the dataset of pairwise comparisons $\{(x_k, y^+_k, y^-_k)\}_{k\in\{1,\dots,K\}}$, where $K$ is the length of the dataset. If we denote the probability of statement $y$ being the response of the LLM to prompt $x$ as $\pi(y|x)$, the LLM can be thought of as a distribution over all possible responses $\mathcal{Y}$. This distribution has been trained on the dataset $\{(x_k, y^+_k, y^-_k)\}_{k\in\{1,\dots,K\}}$. Thus, $\pi( . |x)$ is a function from the preference profile $\{\succ_i\}_{i\in\mathcal{P}}$ to a distribution over $\mathcal{Y}$. Therefore, it is a Probabilistic Social Choice function. To simplify notation, in the rest of the paper we will omit the conditioning on prompt $x$.

Therefore, solving the problem of alignment requires (a) to choose a Probabilistic Social Choice function $\rho$ with desirable properties from Social Choice Theory (e.g. Majority, Condorcet Consistency, Pareto Efficiency, IIA, ...), and (b) to finetune the LLM pushing its distribution as close as possible to that of the Probabilistic Social Choice function $\rho$.

\subsection{RLHF Implements Borda}\label{sec:borda}

There is already an existing connection between current usages of RLHF and Social Choice Theory. In a recent paper, ~\citeauthor{siththaranjan2024distributional} showed that the standard RLHF methods based on the Bradley-Terry model effectively implement the Borda scoring rule (Theorem 3.1~\cite{siththaranjan2024distributional}). For the sake of completeness we provide the full theorem statement and proof in \Cref{appendix:borda_count}.

Since Borda is a well understood Social Choice function, we know that it is not Condorcet consistent. This means that all RLHF methods that aggregates individuals' preferences by emulating Borda may result in some counter-intuitive outcomes. Consider the example in~\Cref{fig:example-image}. A group of five individuals are asked to specify their favourite colour. Two of the five report that they prefer red more than green, and green more than blue (i.e. $\textcolor{red}{R} \succ_i \textcolor{green}{G} \succ_i \textcolor{blue}{B}$ for $i \in \{1,2\}$). Three of the five report they prefer blue more than red, and red more than green (i.e. $\textcolor{blue}{B} \succ_i \textcolor{red}{R} \succ_i \textcolor{green}{G}$ for $i \in \{3,4,5\}$). Applying Borda to this example, the Borda scores for the three alternatives (i.e., binary win counts) are 7 for red, 6 for blue, and 2 for green. Thus, an RLHF trained policy would be biased towards returning red, which seems counterintutive and not necessarily a good reflection of the underlying preferences of the group. This raises the question: \emph{What properties do we want alignment methods for LLMs to support?}

\subsection{Properties for Alignment}\label{sec:properties}

In this section we propose several properties to assess the alignment for LLMs.
These properties are inspired by concepts studied in the Social Choice literature and address concerns that arise when reasoning about aggregation of individuals' preferences~\cite{brandt2016handbook}, while also addressing some of the concerns recently raised in~\citet{casper2023open}.

First, we argue that outcomes like the one shown in \Cref{sec:borda} should be avoided. When \textbf{\textcolor{blue}{B}} is preferred by a majority of the individuals, that is what the LLM should return. In other words, any alignment method should emulate a Social Choice function that is \emph{majority consistent}.

\begin{definition}
A Social Choice function $f$ is \textit{majority consistent} if for all preferences $\{{\succ_i}\}_{i \in \mathcal{P}}$, if
$$\exists y^{\star} \in \mathcal{Y} \text{ s.t. } \# \{i\ \in \mathcal{P} : \forall y \in \mathcal{Y} \setminus \{y^{\star}\}, ( y^{\star} \succ_i y) \}\geq  \frac{\# \mathcal{P}}{2}$$  
then $y^{\star}  = f(\{{\succ_i}\}_{i \in \mathcal{P}}).$
\end{definition}
A Condorcet winner is an alternative that beats every other alternative in a pairwise majority vote.

\begin{definition}\label{def:condorcet_winner}
Alternative $a \in \mathcal{Y}$ is a \textit{Condorcet winner} with respect to preferences $\{{\succ_i}\}_{i\in\mathcal{P}}$ if for all $b \in \mathcal{Y} \setminus \{a\}$, 
$N(a, b) > N(b, a)$, where 
$N(a, b) = \# \{i \in \mathcal{P}: a \succ_i b\}$.
\end{definition}

Clearly a majority winner is a Condorcet winner. Any Social Choice function that returns a Condorcet winner when it exists is called \emph{Condorcet consistent}. It has been argued that a Condorcet winner captures the inherent representativeness of the individuals' preferences and is viewed as a consensus choice~\cite{deCondorcet1785}, Any alignment method that emulates a Condorcet consistent Social Choice function will also best reflect the interests of the population.

Condorcet winners may not always exist. In particular, if there are collection of preferences that induce a cycle, then there is no Condorcet winner. A simple example where this happens is shown in \Cref{tab:rock_paper_scissors}, where there is no clear consensus as to which is the socially preferred colour. We would like an alignment method that can capture this lack of agreement across the individuals, allowing for nuance. In particular, we argue that alignment methods should emulate probabilistic Social Choice functions.
\begin{definition}
Given preferences $\{{\succ_i}\}_{i \in \mathcal{P}}$, a  \emph{probabilistic} Social Choice function, returns a distribution over alternatives $\mathcal{Y}$. 
\end{definition}

\begin{table}[t!] 
\centering
\begin{tabular}{c|ccc}
\textbf{Preference} & \textbf{Ana} & \textbf{Bob} & \textbf{Carla}\\ \hline
\textbf{1st}         & \textcolor{red}{R }& \textcolor{green}{G } & \textcolor{blue}{B }  \\ \hline
\textbf{2nd}        & \textcolor{blue}{B }  & \textcolor{red}{R } & \textcolor{green}{G }    \\ \hline
\textbf{3rd}        & \textcolor{green}{G } & \textcolor{blue}{B }  & \textcolor{red}{R }   \\ \hline
\end{tabular}
\caption{Cyclic preference example.}
\label{tab:rock_paper_scissors}
\end{table}

Finally, we argue that an alignment method should be robust against irrelevant alternatives whenever possible. The property, \emph{independence of irrelevant alternatives} states that the relative ranking of two alternatives should not be effected by the presence or absence of a third, irrelevant alternative.

\begin{definition} 
A Social Choice function \(f\) satisfies \textbf{IIA} if its choice between any two alternatives \(a\) and \(b\) depends \textbf{only} on how individuals rank \(a\) and \(b\) relative to each other, and not on how they rank other alternatives. Formally: $\forall~a,~b~\in~ \mathcal{Y},$ $\quad~\forall~\text{ profiles }~\{\succ_i\}, \{\succ_i'\},$
\[ \text{ if } a \succ_i b \iff a \succ_i' b, \forall i \in \mathcal{P}, \]
then the Social Choice from those two profiles is the same whenever it concerns choosing between \(a\) and \(b\). That is, if in the first profile \(f(\{\succ_i\})\) is \(a\) (or \(b\)), changing only preferences involving alternatives other than \(a\) and \(b\) cannot change whether \(f\) selects \(a\) or \(b\).
\end{definition}

\section{Standard RLHF Does Not Satisfy Desired Properties}\label{sec:problems}

In the previous section we proposed a set of properties that we believe alignment methods for LLMs should exhibit. In this section we show that current RLHF methods based on the Bradley-Terry model do not satisfy any of the properties, thereby raising the questions—previously noted by others \cite{chen2024preferencelearningalgorithmslearn}—regarding their suitability for alignment problems. While this section builds intuition on simplistic examples, in \Cref{sec:results} we further support our findings with experimental results.

To build intuition, in the following we ignore the prompt, and consider a scenario with only three possible options: \textcolor{red}{R}, \textcolor{green}{G} and \textcolor{blue}{B}. Relying on the fact that RLHF emulates Borda count~\cite{siththaranjan2024distributional}, we also assume that the LLM post-trained using standard RLHF gives probability close to one to whichever single-token word had the highest win-rate comparison. 
Of course, in a realistic scenario the LLM would only provide probability close to one if the KL regularization term from the loss is made negligible, either by training for long enough or by giving it a small weight. However, we argue that, if anything, this raises a new concern: through RLHF, a practitioner is aligning the LLM to behave in a middle point between a pretrained model which only cares about what is the probability of the next output (which lacks alignment guarantees) and a model that gives probability one to the token that has the highest win-rate in a preference dataset. 

\noindent\textbf{RLHF is not Majority Consistent nor Condorcet Consistent:}
We showed that RLHF is not Majority consistent in \Cref{sec:borda}. 
Similarly it is not Condorcet consistent. In the example shown in \Cref{fig:example-image}, the Condorcet winner is \textcolor{blue}{B}. This is because, when \textcolor{blue}{B} is compared with \textcolor{red}{R}, three out of the five individual's prefer \textcolor{blue}{B} to \textcolor{red}{R}. Similarly when \textcolor{blue}{B} is compared to \textcolor{green}{G}, three out of the five individuals prefer \textcolor{blue}{B} to \textcolor{green}{G}. However, since RLHF emulates Borda, the resulting policy will be biased towards \textcolor{red}{R}.\\

\begin{table}[t!] 
\centering

\begin{subtable}[t]{0.45\textwidth}
\begin{tabular}{c|ccccc}
\textbf{Preference} & \textbf{Ana} & \textbf{Bob} & \textbf{Carla} & \textbf{Dario}  & \textbf{Eve}\\ \hline
\textbf{1st}         & \textcolor{red}{R} &\textcolor{red}{R} & \textcolor{blue}{B}& \textcolor{blue}{B}& \textcolor{blue}{B}  \\ \hline
\textbf{2nd}        & \textcolor{blue}{B} &  \textcolor{blue}{B} &  \textcolor{red}{R}& \textcolor{red}{R}  & \textcolor{red}{R}  \\ \hline
\end{tabular}
\caption{A simple scenario with just two colours.}\label{tab:IIA_2}
\end{subtable}

\bigskip 

\begin{subtable}[t]{0.45\textwidth}
\begin{tabular}{c|ccccc}
\textbf{Preference} & \textbf{Ana} & \textbf{Bob} & \textbf{Carla} & \textbf{Dario}  & \textbf{Eve}\\ \hline
\textbf{1st}         & \textcolor{red}{R} &\textcolor{red}{R} & \textcolor{blue}{B}& \textcolor{blue}{B}& \textcolor{blue}{B}  \\ \hline
\textbf{2nd}        & \textcolor{green}{G} & \textcolor{green}{G} &   \textcolor{red}{R}  & \textcolor{red}{R}  & \textcolor{red}{R}  \\ \hline
\textbf{3rd}        & \textcolor{blue}{B} &  \textcolor{blue}{B} &   \textcolor{green}{G}& \textcolor{green}{G}& \textcolor{green}{G}  \\ \hline
\end{tabular}
\caption{Introduction of an irrelevant alternative \textcolor{green}{G}.}\label{tab:IIA_3}
\end{subtable}

\caption{Independence of Irrelevant Alternatives example. (a) shows a simple scenario with a clear majority. (b) introduces an irrelevant alternative \textcolor{green}{G} that should not change the preference ranking.}
\label{tab:IIA}
\end{table}

\noindent{\bf RLHF is not Independent of Irrelevant Alternatives:} 
To explain the Independence of Irrelevant Alternatives (IIA), let's consider the simple scenario shown in \Cref{tab:IIA_2}, with a set of preferences over two alternatives, \textcolor{red}{R} and \textcolor{blue}{B}. Clearly, the aggregated preference ranking is that \textcolor{blue}{B} is socially preferred to \textcolor{red}{R}, and standard RLHF would align a model to most likely return \textcolor{blue}{B}. Now imagine that a third alternative \textcolor{green}{G} is introduced (\Cref{tab:IIA_3}). This addition doesn't change the relative ranking of \textcolor{blue}{B} with respect to \textcolor{red}{R} for any individual in the population. If RLHF was independent of irrelevant alternatives, \textcolor{red}{R} would continue to be lower ranked (and thus be assigned lower reward when learning a policy) than \textcolor{blue}{B}.

However, akin to the case of \Cref{fig:example-image}, this is not the case and RLHF would assign highest reward to \textcolor{red}{R}.

\noindent{\bf Cyclic Preferences:}
Collections of preferences that exhibit cycles, such as those shown in \Cref{tab:rock_paper_scissors}, can be challenging. In these cases there is no Condorcet winner, and Borda is unable to distinguish between the alternatives without relying on some tie-breaking method.

 Concretely, due to the stochastic nature of training, RLHF would likely lead to one of the options (say, \textcolor{green}{G}) having slightly higher reward. This would bias the LLM toward that option, even though it is not genuinely superior and a uniform distribution would be more aligned. Regularization techniques like early stopping or KL penalization can lessen this problem, at the expense of keeping the final distribution closer to the original pre-trained model, which is not necessarily aligned.

\section{Using Maximal Lotteries to Align LLMs}
\label{sec:align-maxlot}

Having established the shortcomings of Bradley-Terry based RLHF, the questions becomes "Is there an alternative approach?". We answer in the affirmative. In particular, we argue that a \emph{probablistic} Social Choice function, \textbf{maximal lotteries}, is particularly well suited for alignment of LLMs.

\subsection{Maximal Lotteries}
\label{subsec:definition_max_lott}

Given a set of preferences, a Probabilistic Social Choice function returns a \emph{distribution} over alternatives, called  a lottery. One particular Probabilistic Social Choice function is the maximal lottery~\cite{Krewaras65, fishburn1984probabilistic}. Define $\Delta(\mathcal{Y})$ as the set that contains all distribution (i.e. lotteries) over the options $\mathcal{Y}$.
A maximal lottery,  
$\pi \in \Delta(\mathcal{Y})$, is one that is (weakly) preferred to any other lottery: namely 
\begin{equation}
    \pi^T M \pi^{\prime} \ge 0, \forall \pi^{\prime} \in \Delta(\mathcal{Y}),
    \label{eq:max_lottery}
\end{equation}
where $M$ is the pairwise margin matrix,
 where each entry $M_{ij}$ represents the net margin of voters who prefer $a$ over $ b$:  
  $M_{ij} = N(a, b) - N(b, a)$.

Equivalently, one can view $M$ as the payoffs of a carefully constructed symmetric zero-sum \emph{margin game} where the payoffs are win or loss magnitudes of different pairwise comparisons. The  maximal lottery is, thus,  the mixed maximin (or Nash equilibrium) solution to the game, and can be computed via linear programming in polynomial time \cite{brandl2022analytical}.

Maximal lotteries (ML) exhibit a number of interesting properties.
First, they require little structure to be placed on  voters' preferences since $M$ is computed solely using pair-wise comparisons. This makes them particularly well suited for current LLM alignment processes where preference data typically takes this form.

Second, they are Condorcet-consistent and Majority-consistent, in that alternatives in the support of the maximal lottery are the Condorcet winner. They also provide a level of protection against irrelevant alternatives through both being clone-consistent (see \Cref{appendix:SCT_properties}) and independent of irrelevant alternatives (in a probabilistic sense, see~\citet{brandl2020arrovian,brandl2016consistent} and \Cref{appendix:probabilistic_voting_rule}), and are able to handle cyclic preferences in a principled manner.

Maximal lotteries, in essence, aim to maximize the probability of selecting an alternative that would win in a pairwise majority comparison against any other alternative.  This captures a strong notion of collective preference, prioritizing options that are most likely preferred by a majority of individuals.
If there is a clear winner (i.e. there is a Condorcet winner), then Maximal Lotteries will give probability one to that option. When there is debate among a few of options, Maximal Lotteries will return a distribution over those options.

\subsection{Using Maximal Lotteries to Align LLMs.}

We believe that emulating Maximal Lotteries with LLMs holds significant potential as a solution to the alignment problem, as it has been shown to be the only \emph{probabilistic} Social Choice function that satisfies the key desiderata of Arrow's Impossibility Theorem in a stochastic setting \cite{brandl2020arrovian}. This ensures that the LLM's output respects fundamental Social Choice principles.

The crucial question, then, is how to train an LLM to behave like a Maximal Lottery. This can be achieved with the following objective function:

\begin{theorem}
Let $\mathcal{Y}$ be the set of all possible statements up to a finite maximum length $L$. Let $\pi$ and $\pi^{\prime}$ represent two policies (i.e., LLMs). For two statements $a, b \in \mathcal{Y}$, let $P(a~\succ~b)$ be the probability that a random individual picked uniformly from society prefers $a$ over $b$. Let $P(a~\sim~b)$ be the analogous quantity, but for indifference.

Then, the solution $\pi^*$ to the following maximin optimization problem
\begin{small}
\begin{equation}
\max_{\pi}  \min_{\pi^{\prime}}  \sum_{a \in \mathcal{Y}} \sum_{b \in \mathcal{Y}} \pi(a) \left( P(a \succ b) +\frac{1}{2} P(a \sim b)  \right)\pi^{\prime}(b) 
\label{eq:theorem_1}
\end{equation}
\end{small}
is the Maximal Lottery for the Social Choice problem defined by the set of alternatives $\mathcal{Y}$ and the population's preferences over these alternatives.  (See proof in \Cref{appendix:proof}.)
\end{theorem}

Beyond the properties highlighted earlier, Maximal Lotteries possess other desirable Social Choice characteristics like participation \cite{brandl2019welfare} and reinforcement \cite{brandl2016consistent}.

\subsection{Maximal Lotteries and the  Connection with Nash Learning From Human Feedback}

The objective function presented in Theorem 1 bears a striking resemblance to the optimization process employed in Nash Learning from Human Feedback (NLHF) \cite{munos2023nash}.  NLHF
aims to find a policy $\pi$ that maximizes its expected reward against an adversarial policy $\pi^{\prime}$: 

\begin{equation}
\max_{\pi}  \min_{\pi^{\prime}} \quad   \sum_{a \in \mathcal{Y}} \sum_{b \in \mathcal{Y}} \pi(a) P(a \succ b) \pi^{\prime}(b),
\label{eq:nlhf}
\end{equation}
where $P(a \succ b)$ represents the probability that a human prefers statement $a$ over $b$.  The key difference between this NLHF formulation and our proposed objective function is the term
$\frac{1}{2}P(a \sim b)$,
which accounts for cases where individuals are indifferent between two options.

This difference highlights a crucial aspect of human preferences: indifference. While standard NLHF focuses solely on strict preferences, our formulation acknowledges that individuals may be equally satisfied with multiple options. 

However, in practical scenarios, we often only have access to data reflecting which option a user 
\textbf{selected} in a pairwise comparison, rather than their true underlying preference. Let's then define  $\tilde{P}(a \succ b)$ as the probability that an individual has \textbf{selected} option $a$ when presented with both options $\{a,b\}$ .  This selection probability can be influenced by various factors, including presentation bias (e.g., users might tend to select the first option presented, like in  \citet{craswell2008experimental} and \citet{wang2018position}).  However, under the assumption that individuals facing indifference choose randomly between the options (which we argue is reasonable if we mitigate the bias by randomizing the order of the two sentences in each datapoint), we can show that maximizing the NLHF objective with the \textbf{selection} probability still converges to the Maximal Lottery:
\begin{corollary}
Assume that when individuals are indifferent between two options they are equally likely to select either option in a pairwise comparison, i.e., $\tilde{P}(a~\succ~b) = P(a \succ b) +\frac{1}{2} P(a \sim b).$  
Then, solving $\max_{\pi} \min_{\pi^{\prime}} \quad \sum_{a \in Y} \sum_{b \in Y} \pi(a) \tilde{P}(a \succ b) \pi^{\prime}(b)$ also yields the Maximal Lottery.
\end{corollary}
This corollary shows the robustness of our approach. Even with noisy data, reflecting selection probabilities rather than true preferences, the optimization process can still recover the desirable properties of the Maximal Lottery. It is important to note, however, that position bias in pairwise comparisons should be considered and mitigated.

\section{Experiments}
\label{sec:exp}

In this section, we compare RLHF with algorithms designed to emulate Maximal Lotteries, evaluating their performance across key Social Choice properties. Specifically, we test whether RLHF fails to satisfy majority rule, Condorcet consistency, and independence of irrelevant alternatives (IIA), and whether it struggles with non-transitive aggregate preferences, and conduct the same analysis for Maximal Lotteries. Full implementation details, including hyperparameters and training configurations, are provided in \Cref{appendix:hyperparameters}. We also note that the literature of NLHF has already compared their methods with RLHF algorithms. We provide a summary of their results in \Cref{appendix:previous_NLHF_experiments}

\subsection{Experimental Methodology}\label{sec:exp-method}

To evaluate the performance of Maximal Lotteries against RLHF, we employ synthetic datasets designed to mimic the structure of real-world preference data commonly used in RLHF training.  These synthetic datasets allow for controlled experimentation and enable a precise analysis of the properties discussed in \Cref{sec:problems,sec:align-maxlot}.
Our synthetic datasets consist of triplets:  `<prompt>, <preferred option>, <rejected option>`.  The prompt remains constant across all datasets, and requires the model to choose a favourite colour from three choices, "red", "blue" or "green", which form the set of possible options (alternatives) $\mathcal{Y}$.

To generate a dataset, we first define a population characterized by a probability distribution $P$ over the set of preferences over the alternatives $\mathcal{Y}$. This distribution represents the underlying preferences of the population. For example, consider a population split in 2 groups, A and B, with 60\% of the population belonging to A who prefer ($\textcolor{red}{R} \succ_{\text{A}} \textcolor{blue}{B}\succ_{\text{A}} \textcolor{green}{G}$) and the remainder (B) who prefer ($\textcolor{blue}{B}\succ_{\text{B}}  \textcolor{red}{R}\succ_{\text{B}} \textcolor{green}{G}$), like in \Cref{fig:example-image}. 

We then iteratively generate $2048$ datapoints in three steps: we first sample two distinct alternatives uniformly from $\mathcal{Y}$ without replacement; we then sample an individual from the population $\mathcal{P}$; and finally we determine the preferred and rejected option according to the individual's preference, and record them as a new dataset row.
By varying the preferences, the population distribution $\mathcal{P}$ and the dataset size, we can generate datasets exhibiting different preference patterns. In all our experiments, we sampled 2048  datapoints.

\subsection{Models}

For our experiments we start from three distinct copies of the pretrained Gemma 2 2b model \cite{team2024gemma} \textit{without} instruction tuning. We train one policy (RLHF policy) using Proximal Policy Optimization (PPO) \cite{schulman2017proximal}, as explained in \Cref{subsec:RLHF}, and a second policy (max-lottery policy) using Self-Play Preference Optimization (SPO)~\cite{swamy2024minimaximalist}, which belongs to the family of algorithms that emulate a Maximal Lottery policy. Finally, we use the last copy of the model as the RLHF reward model.

\textbf{RLHF policy (PPO):} We fine-tune the reward model on the synthetic dataset of human preferences described previously and use it to assign a score to every LLM response. This score guides the policy during reinforcement learning.

\textbf{Maximal Lottery policy (SPO)}: This model is optimized using the objective function presented in Theorem 1 (Section 5) with the SPO algorithm on the same human preference dataset used to train the RLHF reward model. The details of the algorithms, the choice of hyperparameters and their implementation can be found in \Cref{appendix:hyperparameters}.

\subsection{Results}\label{sec:results}

This section reports the comparisons of RLHF and Maximal Lotteries based algorithms on the synthetic dataset described in~\Cref{sec:exp-method}, simulating the cases introduced in \Cref{sec:problems}. 

\begin{figure*}[t]
\centering
\includegraphics[width=0.8\textwidth]{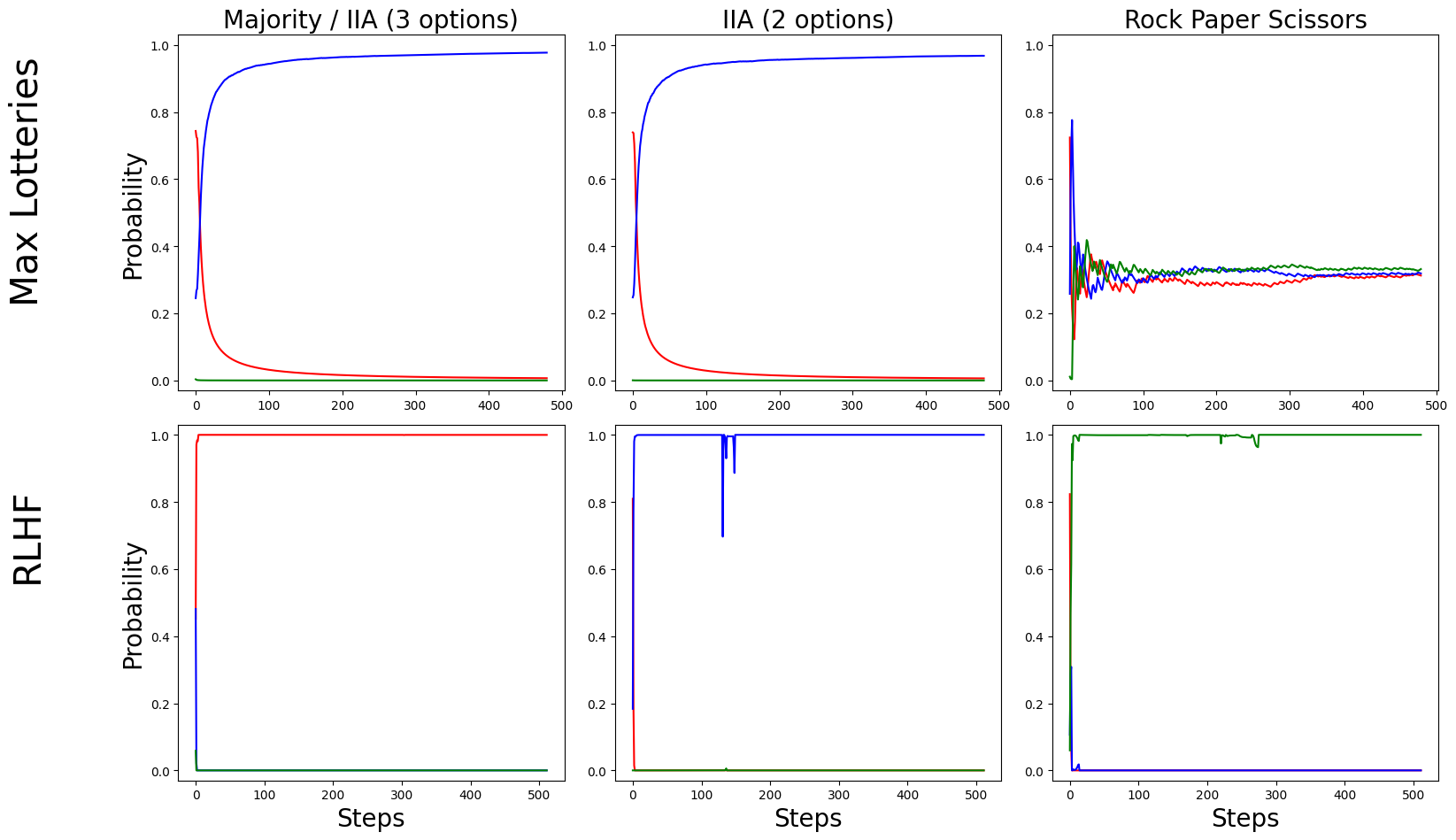}
\caption{
Simulations results \textbf{Left column}: Populations simulate preferences from \Cref{fig:example-image} and \Cref{tab:IIA_3} $(2x(\textcolor{red}{R} \succ \textcolor{green}{G} \succ \textcolor{blue}{B}), 3x(\textcolor{blue}{B} \succ \textcolor{red}{R} \succ \textcolor{green}{G}))$. \textbf{Middle column:} Populations mirror \Cref{tab:IIA_2} preferences $(2x(\textcolor{red}{R} \succ \textcolor{blue}{B}), 3x(\textcolor{blue}{B} \succ \textcolor{red}{R}))$. \textbf{Right column}: Populations exhibit cyclical preferences from \Cref{tab:rock_paper_scissors} $(1x(\textcolor{red}{R} \succ \textcolor{green}{G} \succ \textcolor{blue}{B}), 
 1x(\textcolor{green}{G} \succ \textcolor{blue}{B} \succ \textcolor{red}{R}), 
 1x(\textcolor{blue}{B} \succ \textcolor{red}{R} \succ \textcolor{green}{G}))$. 
}
\label{fig:last_plot}
\end{figure*}

\subsubsection{Experiment 1: Majority and Condorcet}
\label{subsec:exp_prob_majority}

In this experiment (left column of \Cref{fig:last_plot}), we impose a distribution of the population equivalent to that of \Cref{fig:example-image}. The majority alternative is \textcolor{blue}{B}. However, as predicted, RLHF assigns probability close to 1 to the alternative \textcolor{red}{R}. Instead, the maximal lotteries inspired method converges to the preferred alternative \textcolor{blue}{B} with probability close to one.

\subsubsection{Experiment 2: Independence of Irrelevant Alternatives (IIA)}
\label{subsec:exp_prob_IIA}
To evaluate whether RLHF and Maximal Lottery methods respect the Independence of Irrelevant Alternatives (IIA) property, we simulated the scenario described in \Cref{tab:IIA_2,tab:IIA_3}, where preferences among options shift due to the introduction of an irrelevant alternative. Specifically, we used synthetic preference datasets representing two cases: one with three alternatives (\textcolor{red}{R}, \textcolor{blue}{B}, and \textcolor{green}{G}), which coincides with the experiment in \Cref{subsec:exp_prob_majority}, and one with only two alternatives (\textcolor{red}{R} and \textcolor{blue}{B}).

The center column of \Cref{fig:last_plot} (with two alternatives) and the left column (with three alternatives), reveal that RLHF violates the IIA property. Indeed, in the two-alternative scenario, the RLHF-trained policy assigns near-zero probability to \textcolor{red}{R}, favoring the majority winner \textcolor{blue}{B} instead. However, in the three-alternative case, RLHF reverses this decision, giving almost all probability to \textcolor{red}{R}. 

In contrast, the Maximal Lottery approach maintains a stable output distribution across the two scenarios. Regardless of whether \textcolor{green}{G} is included, the probability assigned to \textcolor{red}{R} and \textcolor{blue}{B} remains consistent, close to 1 for \textcolor{blue}{B}, showing that Maximal Lottery methods satisfy the IIA property.

\subsubsection{Experiment 3: Cyclic Preferences}
\label{subsec:exp_prob_2}
In this experiment, we impose the population distribution of \Cref{tab:rock_paper_scissors}. As it can be seen in the right column of \Cref{fig:last_plot}, the maximal lotteries inspired method converges to an LLM that returns each of the colours \textcolor{blue}{B}, \textcolor{red}{R} and \textcolor{green}{G} approximately 33\% of the time.

In contrast, the policy trained with RLHF converges to a policy that returns one arbitrary colour (in this particular simulation \textcolor{green}{G}) with probability one.

\section{Related Work}

This work builds upon several areas at the intersection of AI alignment and Social Choice Theory. Traditional approaches such as Reinforcement Learning from Human Feedback (RLHF) have become the de-facto standard \cite{christiano2017deep, stiennon2020learning} to finetune LLMs. While RLHF has proven effective for guiding LLMs, recent studies have highlighted its limitations \cite{siththaranjan2024distributional, casper2023open, ge2024axioms}.

Recent research has explored the application of Social Choice Theory to address the AI alignment problem. Papers such as  \cite{ge2024axioms, dai2024mapping, mishra2023ai, conitzerposition} argue for viewing alignment as a Social Choice Theory problem, which allows the application of well-established Social Choice functions to aggregate human preferences.

Recent results identify Maximal Lotteries as the unique probabilistic voting system satisfying Arrow’s axioms \cite{brandl2020arrovian}, which has motivated its use in different areas of Machine Learning \cite{lanctot2023evaluating}.

Finally, this work also connects with the emerging field of Nash Learning with Human Feedback (NLHF) \cite{munos2023nash, calandriello2024human, swamy2024minimaximalist}, which proposes alternatives to RLHF based on an optimzation process inspired by Game Theory. 

\section{Limitations and future work}

Our proposed framework, while offering a robust theoretical foundation for aligning LLMs with aggregate human preferences, faces some limitations that require further investigation.

A central challenge lies in the estimation of $P(a\succ b | x)$, particularly in two key points: 1) what do we mean when we say that an individual $i$ prefers $a$ to $b$ ($a\succ_i b)$; and 2) how do we capture that the preferences depend, not only on the prompt $x$, but on the context. 

On the first point, what is the correct interpretation that an individual prefers an option $a$ with respect to another $b$? How realistic is it to assume that it is possible to estimate the preferences of an individual by showing them pairs of sentences, that is the standard practice nowadays? Are there better ways to infer preferences? On this issue, we point to the reader to \citet{gabriel2020artificial} for an extensive discussion. Microeconomic theory and Industrial Organization theory has a history of attacking similar problems and could be a promising avenue to solve them.

Secondly, the appropriateness of a response can vary significantly depending on the context in which a conversation takes place. An answer that is perfectly acceptable in a comedy show might be entirely inappropriate in a professional setting. Therefore, it is crucial to explore methods that allow large language models (LLMs) to incorporate contextual information when generating responses. Developing strategies to enhance context awareness in LLMs is an important step toward more reliable and nuanced AI interactions. For a more in-depth discussion on the theory of appropriateness, we refer readers to \citet{leibo2024theory}.

Another important avenue for future work is the development of an online version of our approach that continuously updates and adapts to changes in societal preferences. Human values and societal norms evolve over time, and a static alignment approach may become outdated or fail to reflect current ethical considerations. An online adaptation mechanism would allow the model to integrate new preference data dynamically, ensuring that its responses remain aligned with contemporary views while avoiding abrupt shifts that could lead to instability or exploitation by adversarial actors. We believe that developing online voting mechanisms that approximate maximal lotteries, such as those explored in \cite{brandl2024natural}, is a promising direction for achieving this goal.

Addressing these previous points is crucial for realizing the full potential of our framework. By combining rigorous Social Choice principles with advanced machine learning techniques, we can strive to develop LLMs that are more reliably and ethically aligned with the diverse values and preferences of humanity.

\section{Conclusion}

This paper examines the limitations of RLHF in aligning LLMs with aggregate human preferences, demonstrating its vulnerability to violations of key Social Choice principles, and proposing an alternative framework grounded in Maximal Lotteries. We establish a formal connection between this optimal voting system, known to be the only probabilistic voting system that circumvents Arrow's impossibility theorem, and Nash Learning from human feedback (NLHF) algorithms, offering a practical path for training LLMs that robustly reflect collective human preferences. Our experimental results confirm that methods that emulate Maximal Lotteries, like NLHF and variantes, can overcome the shortcomings of RLHF, yielding LLM whose responses better align with the majority's will. This includes supporting the preferences of the majority, providing principled ways of handling non-transitivities in the preference data, and independence of irrelevant alternatives. The shift from simple reward maximization to a framework rooted in the rich theoretical foundations of Social Choice Theory promises a more nuanced and robust approach to aligning LLMs with human values, ultimately contributing to the development of AI systems that truly serve humanity's best interests.


\section*{Impact Statement}

Ensuring that AI systems are aligned with diverse human values and preferences is critical for the future of society. The growing influence of AI in decision-making processes, from healthcare to education, emphasizes the importance of considering and valuing everyone's preferences. By integrating techniques that emulate Maximal Lotteries, we provide a robust framework for AI alignment, addressing key limitations of existing methods, such as RLHF. However, achieving true alignment also requires accurately estimating individual preferences and tackling challenges like reward hacking. Additionally, the datasets used to estimate these preferences must be created with a representative sample of the population to ensure fairness and inclusivity. If these issues are not tackled, we could end up aligning LLMs to the wrong set of values and preferences, which could have harmful unintended consequences with highly capable AI systems.


\appendix

\bibliographystyle{icml2025}
\bibliography{bibliography}

\section{Appendix}\label{appendix:appendix}

\subsection{Arrow's Impossibility Theorem}
\label{appendix:arrows_impossibility_theorem}

In this section we will discuss Arrow's Impossibility Theorem \cite{arrow1950difficulty}, arguably the most fundamental result in Social Choice Theory. In this theorem, it is shown that, if $\#\mathcal{Y} \geq 3$, there is no \textbf{deterministic} voting system such that it satisfies three basic properties: Independence of Irrelevant Alternatives, Pareto Efficiency and Non-dictatorship. 

Throughout this paper we have used Social Choice functions for ease of exposition. However, this theorem is usually expressed using Social Welfare functions (SWF) $F$, i.e. maps from preference profiles $\{{\succ_i}\}_{i \in \mathcal{P}}$ to a ranking, $\succ_S$. Note, however, that any Social Welfare Function implicitly defines a Social Choice function that returns the top ranked option.

The IIA property was explained in \Cref{subsec:social_choice}, but will be re-expressed for SWFs. The last two properties will be introduced later in this section. 

\subsubsection{IIA - SWF version}

\textbf{Definition}
A social welfare function \(F\) satisfies \textbf{IIA} if: $\forall~a,~b~\in~ \mathcal{Y},$ $\quad~\forall~\text{ profiles }~\{\succ_i\}, \{\succ_i'\},$
\[ \text{ if } a \succ_i b \iff a \succ'_i b, \forall i \in \mathcal{P}, \]
\[ \text{ then } a \succ_s b \iff a \succ_s' b,\] where $\succ_s = F(\{{\succ_i}\}_{i \in \mathcal{P}})$ and $\succ'_s = F(\{{\succ'_i}\}_{i \in \mathcal{P}})$

\subsubsection{Pareto Efficiency}
Intuitively, if everyone prefers outcome $x$ to $y$, then collectively we should also prefer $x$ over $y$. That property is captured by Pareto Efficiency. 

\textbf{Definition (Pareto Efficiency for SWFs):} A Social Welfare function $F$ is \textit{Pareto efficient} if for any preference profile  $\{{\succ_i}\}_{i \in \mathcal{P}}$ and for any two alternatives $x, y \in \mathcal{Y}$, if  $x \succ_i y$ for all $i \in \mathcal{P}$, then $x \succ_S y$, where  $\succ_S = F(\{{\succ_i}\}_{i \in \mathcal{P}})$.

Example: when choosing between chocolate and vanilla, if everyone in a group prefers chocolate ice cream to vanilla, choosing chocolate would be Pareto Efficient. Choosing vanilla would not be, as everyone could be made better off by switching to chocolate.

\subsubsection{Non dictatorship}
This property formalizes the intuitive idea that a dictator, an individual that makes all collective decisions, is not a desirable form of making choices.

\textbf{Definition (non-dicatorship for SWFs):} A Social Welfare function $F$ satisfies \textit{non-dictatorship} if there is no individual $i$ (the dictator) such that for any preference profile  $\{{\succ_i}\}_{i \in \mathcal{P}}$  $\forall x, y \in \mathcal{Y}$,   $x \succ_i y$ if and only if $x \succ_S y$, where  $\succ_S = F(\{{\succ_i}\}_{i \in \mathcal{P}})$.

\subsection{Social Choice Theory properties}\label{appendix:SCT_properties}

In this section, we will list other sets of important properties and paradoxes in Social Choice Theory. \\

\textbf{Monotoniticy \cite{smith1973aggregation, felsenthal2011review}}: A Social Choice function satisfies monotonicity if, whenever $x$ is elected under a distribution of voters' preferences, $x$ keeps being elected if some voters increase their support for $x$ (i.e. $x$ moves higher up in their ranking) keeping everything else constant.

\textbf{No show paradox \cite{fishburn1983paradoxes, felsenthal2011review}}
A voter could obtain a better outcome by not participating in the voting. 

\textbf{Strategic voting paradox \cite{gibbard1973manipulation, felsenthal2011review}} A voter may obtain a better outcome if they strategically lie when reporting their preferences. 

\textbf{Clone-consistency \cite{tideman1987independence}}
This property is a subcase of IIA. The addition of a clone to the set of options (i.e. an option $y_c$ which is quite similar to another option $y$ in the set $\mathcal{Y}$ and thus is placed side by side in the rankings of all voters) should not change the chosen candidate of the Social Choice function. 

\subsection{Maximal Lotteries and Arrow's theorem}\label{appendix:probabilistic_voting_rule}

This section will effectively be a summary of some of the definitions and axioms from \citet{brandl2020arrovian}.
 
So far, individuals have a preference over the options $\mathcal{Y}$. In this section, we will extend those to preferences over distributions (i.e. lotteries).

Let $\mathcal{Y}$ be a finite set of alternatives, and let $\Delta$ be the set of all probability distributions over $\mathcal{Y}$. An element $p \in \Delta$ represents a lottery over alternatives in $\mathcal{Y}$. Call $\mathcal{P}=\{1, \dots, n\}$ the set of voters, and each voter has a preference relation $\succeq_i$ over $\Delta$.

Given $p\in\Delta$, for $i\in \mathcal{P}$ define: 
\begin{itemize}
    \item $U_i(p) = \{q\in \Delta\colon q\succ_i p\}$ is the \emph{strict upper contour set} of $p$ 
    \item  $L_i(p) = \{q\in \Delta\colon p\succ_i q\}$ is the \emph{strict lower contour set}  of $p$ 
    \item $I_i(p) = \{q\in\Delta\colon p\sim_i q\}$ is the \emph{indifference set} of $p$. 
\end{itemize}

For $Z\subseteq \Delta$, ${\succeq}|_Z = \{(p,q)\in {\succeq}\colon p,q\in Z\}$ is the preference relation $\succeq$ restricted to outcomes in $Z$.

\subsubsection{Assumptions on Individual Preferences}

Each individual's preferences $\succeq_i$ must satisfy:

\textbf{Continuity:} Intuitively, if $p \succ_i q$, then small changes in $p$ or $q$ will not reverse the preference. More formally, $U_i(p) \text{ and } L_i(p) \text{ are open} \text.$

\textbf{Convexity:} Intuitively, if $p \succ_i q$, then any mixture $r = \lambda p + (1-\lambda) q$ (for $0 < \lambda < 1$) is also preferred to $q$. More formally:

	$U_i(p), L_i(p), U_i(p)\cup I_i(p)\text{, and } L_i(p)\cup I_i(p) \text{ are convex}\text.$

\textbf{Symmetry:} Intuitively, as explained by \cite{fishburn1984ssb}, "the degree to which $p$ is preferred to $q$ is equal in magnitude (but opposite in sign) to the degree to which $q$ is preferred to $p$". More formally: 

$\forall p,q,r \in \Delta , \forall \lambda \in (0,1)$
\begin{equation}
\begin{aligned}
	\text{if } q\sim\nicefrac{1}{2}\,p+\nicefrac{1}{2}\,r \text{ and } p \lambda r \sim \nicefrac{1}{2}\,p+\nicefrac{1}{2}\,q\\
	\text{ then } r \lambda p \sim \nicefrac{1}{2}\,r+\nicefrac{1}{2}\,q\text.
\end{aligned}
\end{equation}

where $a \lambda b := \lambda a  + (1- \lambda ) b , \forall a,b\in \Delta$.

\subsubsection{Arrovian Properties}

A Social Welfare function (SWF) $F$ maps individual preferences $(\succeq_1, \dots, \succeq_n)$ to the collective preference $\succeq$. In this section, we will describe a generalization of Arrows Impossibility Theorem's main properties.

\textbf{Independence of Irrelevant Alternatives (IIA) - \citet{brandl2020arrovian} version}: 
Let $Z\subseteq \mathcal{Y}$ be a subset of the original options and $\Delta_Z$ be the set of loteries over $Z$. A SFW $F$ satisfies IIA if and only if, for any two preference profiles $\{\succeq_i\}_{i \in \mathcal{P}}$ and $\{\succeq'_i\}_{i \in \mathcal{P}}$, if 
$$\forall i\in \{1, \dots, n\} \left({\succeq_i}|_{\Delta_{Z}} = {\succeq'_i}|_{\Delta_{Z}}\right)$$
then 
$$F(\succeq_1,\dots, \succeq_n)|_{\Delta_{Z}} = F(\succeq'_1,\dots, \succeq'_n)|_{\Delta_{Z}}$$

\textbf{Pareto Efficiency - \citet{brandl2020arrovian} version:} Let $\succeq = F(\succeq_1,\dots, \succeq_n)$. We say that $F$ is Pareto Efficient if, whenever every individual prefers $p$ to $q$ ($p \succeq_i q$ for all $i$), then $p \succeq q$ collectively. If, additionally, there exist individuals $i\in \mathcal{P}$ such that they strictly prefer $p$ to $q$ ($\exists i\in \mathcal{P} (p \succ_i q)$), then $p \succ q$. 

\textbf{Anonymity:} The SWF treats all individuals symmetrically (no individual’s preferences are given special weight). Note that this is stronger than non-dictatorship. 

More formally: 
Let $\pi$ be a permutation of the voters $\mathcal{P}$. Then a SWF $F$ satisfies Anonymity if
$$F(\succeq_1,\dots, \succeq_n) = F(\succeq_{\pi(1)},\dots, \succeq_{\pi(n)} )$$
\textbf{Maximal Lotteries}: It has been proved that under Continuity, Convexity, Symmetry and other technical assumptions, there exist a unique SWF $F$ that satisfies IIA, Anonymity and Pareto Efficiency \cite{brandl2020arrovian}. The Probabilistic Social Choice function $\rho$ that outputs the first lottery of the ranking returned by SWF $F$ is precisely the Maximal Lottery.

\subsection{Relevance of Condorcet and Majority in Text data: Smith sets}\label{appendix:smith_sets}

A potential objection to applying Condorcet criteria to LLMs is the sheer scale of the output space.  With all possible statements up to a certain length as alternatives, it seems unlikely that a single statement would emerge as a Condorcet winner, preferred by a majority over every other possible statement.  However, this vastness doesn't negate the relevance of Condorcet principles. Instead, we can consider the concept of Smith Sets, which offers a generalization of Condorcet winners.  A Smith Set is the smallest non-empty set of alternatives such that every alternative within the set beats every alternative outside the set in a pairwise majority vote. Critically, a Smith Set always exists.

Imagine, for instance, an LLM responding to the prompt "Summarize the French Revolution". While no single summary might be universally preferred, a Smith Set could comprise a collection of summaries deemed superior by a majority to any summary outside this set. This set would capture a core of high-quality summaries reflecting the majority's preferences, even if individuals disagree on the nuances within the set. As long as the output is contained in the Smith set, that would be adequate.

Therefore, while a strict Condorcet winner might be rare in the LLM context, focusing on properties like Condorcet consistency and Majority, which are closely related to Smith Sets, ensures that the LLM prioritizes outputs preferred by a majority to a significant portion of alternatives, thus aligning with a robust notion of collective preference. 

\subsection{Proof of the main theorem}\label{appendix:proof}
In this section we will prove the main theorem of the paper. The notation has been slightly changed to make the proof easier to follow (we substitute $a$ with $y_i$ and $b$ with $y_j$).

\begin{theorem*}
Let $\mathcal{Y}$ be the set of all possible statements with a number of tokens smaller than a predetermined maximum length $L$. Let $\pi$ and $\pi^{\prime}$ represent two policy LLMs. For two statements $y_i, y_j \in \mathcal{Y}$, let $P(y_i \succ y_j)$ be the probability that a random individual picked uniformly from society prefers $y_i$ over $y_j$. Let $P(y_i \sim y_j)$ be the analogous quantity, but for indifference.

Then, the solution $\pi^*$ to the following maximin optimization problem:
{\small
$$\max_{\pi}  \min_{\pi^{\prime}}  \sum_{y_i \in \mathcal{Y}} \sum_{y_j \in Y} \pi(y_i) \left( P(y_i \succ y_j) +\frac{1}{2} P(y_i \sim y_j)  \right)\pi^{\prime}(y_j) $$
}is the Maximal Lottery for the Social Choice problem defined by the set of alternatives $Y$ and the population's preferences over these alternatives.
\end{theorem*}

\begin{proof}
  First, some notation. Define : 

\begin{itemize}
\item $n$ is the amount of elements in the population.
\item $m$ is the amount of elements in $\#\mathcal{Y}$.
\item $N$ as the matrix that indicates \textbf{number} of people who prefer statement $y_i$ to $y_j$ $N:=(\# \{k:y_i\succ_{k} y_j \})_{i,j}$  
\item  $E$ as the matrix that indicates \textbf{number} of people who are indifferent between statement $y_i$ to $y_j$ $E:=(\# \{k:y_i\sim_{k} y_j \})_{i,j}$  
\item The Margin matrix $M:=N-N^T$ 
\item  $\tilde{N}$   as the matrix that indicates \textbf{proportion} of people who prefer statement $y_i$ to $y_j$ , i.e. $\tilde{N}= \left( P(y_i \succ y_j)\right)_{i,j}=N/n$.
\item  $\tilde{E}$   as the matrix that indicates \textbf{proportion} of people who are indifferent between statement $y_i$ and $y_j$ , i.e. $\tilde{E}= \left( P(y_i \sim y_j)\right)_{i,j}=N/n$. 
\item The proportion margin matrix $\tilde{M}=M/n$ 
\end{itemize}

Note how $N$, $E$, $M$, $\tilde{N}$, $\tilde{E}$ and $\tilde{M}$ are all matrices of shape $m\times m$.

Define also the matrix of all ones as: 
\begin{equation*} 
\mathbb{J}_{m} =
\begin{pmatrix}
1 & \cdots & 1 \\
\vdots & \ddots & \vdots \\
1 & \cdots & 1
\end{pmatrix}_{m \times m}
\end{equation*}
Observe that
\begin{align*}
P(y_j \succ y_i) &= 1- P(y_i \succeq y_j) \\
&= 1- P(y_i \succ y_j) -  P(y_i \sim y_j)
\end{align*}
Therefore 
\begin{align}
\tilde{N}^T &=  \mathbb{J}_{m} - (P(y_i \succ y_j))_{i,j} - (P(y_i \sim y_j))_{i,j} \\
 &=  \mathbb{J}_{m} -\tilde{N} - \tilde{E}
\end{align}
Thus, 
\begin{align}
\tilde{M} &=\tilde{N} - \tilde{N}^T\\
&= \tilde{N} -  ( \mathbb{J}_{m} -\tilde{N} - \tilde{E}) \\
&= 2 \tilde{N} -   \mathbb{J}_{m} + \tilde{E} \label{matrix_notation}
\end{align}
Finally, note that for any probability vectors $p,q\in \Delta(\mathcal{Y})$ , then
\begin{align}
p^T  \mathbb{J}_{m} q &= \langle\ (\sum_{i=1}^m p_i, ..., \sum_{i=1}^m p_i), q \rangle \\
&= \langle\ (1, ..., 1), q \rangle \\
&= \sum_{i=1}^m q_i \\
&= 1  \label{matrix_ones}
\end{align}
A lottery $\pi$ is maximal if  $\pi^T M \geq 0$,. In other words, no other lottery $\pi^{\prime}$ is preferred by an expected majority of voters ($\pi^T M \pi^{\prime} \geq 0$). 

The maximal lottery can also be calculated as the solution from the following optimization problem: 
 $$\max_{\pi}  \min_{\pi^{\prime}} \quad  \pi^T M \pi^{\prime} $$
 From there, 
 {\small
\begin{align*}
\pi^{*} &= \arg \max_{\pi}  \min_{\pi^{\prime}} \quad  \pi^TM \pi^{\prime}  \\
&= \arg \max_{\pi}  \min_{\pi^{\prime}} \quad  \pi^T\frac{M}{n} \pi^{\prime}  && \tag{a} \\
&=  \arg\max_{\pi}  \min_{\pi^{\prime}} \quad  \pi^T\tilde{M} \pi^{\prime}  && \tag{b} \\
&=  \arg\max_{\pi}  \min_{\pi^{\prime}} \quad  \pi^T( 2 \tilde{N} -   \mathbb{J}_{m} + \tilde{E})  \pi^{\prime}  && \tag{c} \\
&=  \arg\max_{\pi}  \min_{\pi^{\prime}} \quad   2 \pi^T\tilde{N} \pi^{\prime}-    \pi^T\mathbb{J}_{m}\pi^{\prime} +  \pi^T\tilde{E}\pi^{\prime}    \\
&=  \arg\max_{\pi}  \min_{\pi^{\prime}} \quad   2 \pi^T\tilde{N} \pi^{\prime}-    1 +  \pi^T\tilde{E}\pi^{\prime}    && \tag{d} \\
&=  \arg\max_{\pi}  \min_{\pi^{\prime}} \quad   2 \pi^T\tilde{N} \pi^{\prime} +  \pi^T\tilde{E}\pi^{\prime}    && \tag{e} \\
&=  \arg\max_{\pi}  \min_{\pi^{\prime}} \quad   \pi^T\tilde{N} \pi^{\prime} +  \frac{1}{2}\pi^T\tilde{E}\pi^{\prime}      && \tag{f} \\
&=  \arg\max_{\pi}  \min_{\pi^{\prime}} \quad   \pi^T (P(y_i \succ y_j))_{i,j}  \pi^{\prime}\\
& \hspace{2.4cm} +  \frac{1}{2}\pi^T(P(y_i \sim y_j))_{i,j} \pi^{\prime}      \\
&=  \arg\max_{\pi}  \min_{\pi^{\prime}} \quad   \sum_{y_i \in Y} \sum_{y_j \in Y} \pi(y_i)  P(y_i \succ y_j)\pi^{\prime}(y_j)  \\
& \hspace{1.8cm} +  \sum_{y_i \in \mathcal{Y}} \sum_{y_j \in \mathcal{Y}} \pi(y_i)  \frac{1}{2} P(y_i \sim y_j) \pi^{\prime}(y_j)  && \tag{g}
\end{align*}
}
\begin{itemize}
\item (a): Dividing by constant does not change solution\\
\item (b): Change notation\\
\item (c): Using \Cref{matrix_notation}\\
\item (d): Using \Cref{matrix_ones}\\
\item (e): Subtracting constant does not change anything\\
\item (f): Dividing by constant does not change solution\\
\item (g): Expand terms
\end{itemize}

The last term can easily be rearranged to get our result. This ends our proof. 
\end{proof}

\subsection{Random dictatorships and pretrained LLMs}
\label{appendix:random_dictatorships}

In this section, we highlight a connection between the behavior of pretrained LLMs and a well-known probabilistic Social Choice function: random dictatorships.

A random dictatorship selects a single individual from the population at random and implements their top-ranked choice \cite{gibbard1977manipulation}. Pretrained LLMs, which approximate the probability of the next token based on the distribution of text in their training data, can be seen as implicitly implementing a form of random dictatorship. In this view, the "voters" are the users who contributed to the dataset, and their influence is weighted by the volume of text they generated. This suggests that before fine-tuning, LLMs may already reflect an aggregation of individual preferences, albeit in a way that is biased by data distribution rather than designed to satisfy desirable Social Choice properties.

\subsection{The multiple definitions of IIA}\label{appendix:definitions_IIA}

Recent results in the literature have pointed out that RLHF \textit{satisfies} IIA \cite{xu2023rlhf}. This might be confusing for some readers, as we have precisely argued that  RLHF \textit{does not satisfy} IIA. The reason is that, regretably, the concept of IIA is used to refer to very different properties in different fields. For clarification, we point to the reader towards \citet{ray1973independence}. 

In \citet{xu2023rlhf}, the issue with IIA is raised intuitively in their paper in the following way: assume that individuals have to choose what they prefer between \textit{cats}, \textit{felines} and \textit{dogs}. In this example, \textit{cats} and \textit{felines} are synonyms. Thus, if we add the word \textit{feline} to our set of possible considerations, that should only affect the probability of returning the word \textit{cat}, but should not affect the probability of returning the word \textit{dog}. That is: $\frac{\mathbb{P}(Y=\textit{dog} \mid \mathcal{Y}=\textit{\{dog, cat\}})}{\mathbb{P}\left(Y= \textit{cat} \mid \mathcal{Y}=\textit{\{dog, cat\}}\right)}=\frac{\mathbb{P}\left(Y=\textit{dog} \mid \mathcal{Y}=\textit{\{dog, cat, feline\}}\right)}{\mathbb{P}\left(Y\in \textit{\{cat, feline\}} \mid \mathcal{Y}=\textit{\{dog, cat, feline\}}\right)}.$

More formally, let $\mathcal{M}$ be the set of all messages and  $\mathcal{X}, \mathcal{X}^{\prime}\subseteq \mathcal{M}$ are some possible subsets of that set of words. 

Let $x \in \mathcal{M}$ be a message. 
Let $\mathbb{P}(Y=x \mid \mathcal{Y}=\mathcal{X})$ be the proportion of individuals who prefer the message $x$ over all other messages in the set $\mathcal{X}$

Then, in \citet{xu2023rlhf}, the IIA definition is inspired by the definition from \citet{luce1959individual}: IIA-Luce means that for all messages $x, x^{\prime} \in \mathcal{M}$ and sets $\mathcal{X}$ and $\mathcal{X}^{\prime}$ such that $x, x^{\prime} \in \mathcal{X} \cap \mathcal{X}^{\prime}$,
$$
\frac{\mathbb{P}(Y=x \mid \mathcal{Y}=\mathcal{X})}{\mathbb{P}\left(Y=x^{\prime} \mid \mathcal{Y}=\mathcal{X}\right)}=\frac{\mathbb{P}\left(Y=x \mid \mathcal{Y}=\mathcal{X}^{\prime}\right)}{\mathbb{P}\left(Y=x^{\prime} \mid \mathcal{Y}=\mathcal{X}^{\prime}\right)}
$$

It is worth mentioning that this property is connected to the property of composition consistency, which has been recently shown to be a property of Maximal Lotteries \cite{brandl2016consistent}.

\subsection{Experimental details and hyperparameters}\label{appendix:hyperparameters}

The two prompts used in our experiments were the following:

\begin{verbatim}

Prompt (IIA-2 options)   
"""
Q: What is your favorite color from
the options red and blue?
answer in the format 'My favourite 
color is the color red.' or 
'My favourite color is the color
blue.' and say nothing else after 
that. \n"
A: My favourite color is the color
"""

Prompt (all other experiments)    
"""
Q: What is your favorite color from
the options red, blue and green?
answer in the format 'My favourite
color is the color red.' 'My favourite
color is the color blue.' or 
'My favourite color is the color
green.' and say nothing else after
that. \n
A: My favourite color is the color
"""
\end{verbatim}

The distributions over the preferences of the population were defined in ways similar to the following example:
\begin{verbatim}
rankings = {
    "voter_0": [R, G, B],
    "voter_1": [G, B, R],
    "voter_2": [B, R, G]
}
# Example probabilities
p = [0.33, 0.33, 0.34] 
\end{verbatim}

All three copies of the Gemma model were trained using LoRA (Low-Rank Adaptation of Transformers) \cite{hu2021lora} with the following configuration:
\begin{itemize}
    \item \textbf{Rank (r):} 8
    \item \textbf{Alpha:} 32
    \item \textbf{Dropout:} 0.1
\end{itemize}

\paragraph{Maximal Lottery Policy (SPO):}
See \Cref{alg:SPO} for a pseudocode implementation of SPO. For the Maximal Lottery experiments, the following hyperparameters were used to train the policy using the SPO:
\begin{itemize}
    \item \textbf{RL step of SPO:} we use the PPO algorithm
    \item \textbf{Epochs:} 30
    \item \textbf{Batch size:} 128
    \item \textbf{Mini-batch size:} 32 (split from the main batch)
    \item \textbf{Learning rate:} $1 \times 10^{-4}$
    \item \textbf{Value function coefficient (vf\_coef):} 0.0
    \item \textbf{Initial KL coefficient (init\_kl\_coef):} 0.0
    \item \textbf{Gamma patience:} 0.0 (used to estimate the value function)
    \item \textbf{Entropy coefficient:} Increased to ensure exploration during training.
    \item \textbf{Epochs:} 30
    \item \textbf{Dataset size:} $2^{11}$ (2048) data points.
\end{itemize}

To help with the training, we enforce that 10\% of the batch is a uniform sample of the three colour words " red", " green", " blue".

\paragraph{Preference Function:}
The preference function used for the Maximal Lottery setup returns the percentage of voters who prefer one alternative over another in the dataset out of the three colours. In the edge scenarios, we explicitly enforce the following outputs:
\begin{itemize}
    \item If two alternatives are equal, it assigns a preference score of 0.5.
    \item If one alternative is missing from the dataset, it assumes the present alternative is preferred (score of 1.0).
    \item If both alternatives are missing, it assigns a preference score of 0.5.
\end{itemize}

\begin{algorithm}[tb]
   \caption{SPO algorithm implementation}
   \label{alg:SPO}
\begin{algorithmic}
   \STATE {\bfseries Input:} Iterations $T$, Preference fn. $P$, Num. samples $k \geq 2$, Fix prompt $x$
   \STATE {\bfseries Output:} Trained policy $\pi$
   \STATE Initialize $\pi_1 \in \Pi$.
   \FOR{$i=1$ {\bfseries to} $T$}
   \STATE $s_i=x$
   \STATE Sample $a_{1: k} \sim \pi_t(s_i)$
   \STATE Compute $r_i=\frac{1}{k-1} \sum_{j \neq i}^k P\left(a_i \succ a_j\right)$.
   \STATE $\mathcal{D} =\left\{\left(s_i, a_i, r_i\right)\right\}_{i \in[k]}$ 
   \STATE $\pi_{t+1} \leftarrow \operatorname{RL-PPO}\left(\pi_t, \mathcal{D}\right)$.
   
   \ENDFOR
    \STATE {\bfseries Return:} uniform mixture of $\pi_{1: T}.$
   
\end{algorithmic}
\end{algorithm}

\paragraph{Reward Model for RLHF:}
We left the default hyperparameter configurations of the library \textit{trl} (version 0.10.1), except for hyperparameter \textit{center\_rewards\_coefficient} which is set to 0.01. We trained the reward for 3 epochs. 

\paragraph{RLHF Policy Optimization (PPO):}
The policy for RLHF was trained using the Proximal Policy Optimization (PPO) algorithm with the following hyperparameters:
\begin{itemize}
    \item \textbf{Epochs:} 4
    \item \textbf{Batch size:} 16
    \item \textbf{Learning rate:} $5 \times 10^{-4}$
    \item \textbf{Value function coefficient (vf\_coef):} 0.01
    \item \textbf{Initial KL coefficient (init\_kl\_coef):} 0.0 (no entropy regularization in this experiment)
\end{itemize}

These configurations were chosen to ensure fair comparison.

\subsection{RLHF emulates Borda Count}
\label{appendix:borda_count}

In this section, Theorem 3.1 of \cite{siththaranjan2024distributional} is replicated for completeness. In this theorem, the authors prove that RLHF implicitly behaves like the Borda Count Social Choice function. Our version of the theorem has slighly different notation and lest terms (we ignore regularization), but the conclusion is the same. 

WLOG, the definition of Borda Count will be slightly modified to make the proof easy to follow. Rather than being just the sum of pairwise  victories  over  other  candidates, in this section it is defined as the sum of pairwise victories  over  other  candidates \textit{divided by the number of voters}. Thus, it can be expressed as \Cref{eq:new_borda_count}.

\begin{theorem}[BTL Identifies Borda Count]
\label{thm:BTL_borda_equivalence}
Let $\mathcal{A}=\{a_1,a_2,\dots,a_m\}$ be a finite set of alternatives. Suppose for each ordered pair $(a,b)$ we have an empirical probability $p(a,b)$ representing the fraction of annotators (in the limit of infinite data) who strictly prefer $a$ to $b$, with $p(a,b)+p(b,a)=1$. Define the Borda count of an alternative $a$ as
\begin{equation}
\text{BordaCount}(a)\;=\;\sum_{c \neq a}\;p(a,c).
\label{eq:new_borda_count}
\end{equation}
Now consider training a scalar reward function $r:\mathcal{A}\to\mathbb{R}$ under the Bradley--Terry--Luce (BTL) model via maximum-likelihood on pairwise comparisons.  In the limit of infinite data, the resulting $r(a)$ satisfies:
\[
r(a) > r(b)
\quad\Longleftrightarrow\quad
\text{BordaCount}(a)>\text{BordaCount}(b).
\]
That is, $r(\cdot)$ orders the alternatives exactly by their Borda counts.
\end{theorem}

\begin{proof}
\textbf{1.\; Setup and notation.}
We have pairwise comparison data indicating that $a$ beats $b$ with empirical probability $p(a,b)$. The Bradley--Terry--Luce model posits
\[
\Pr[\text{``$a$ preferred over $b$''}] \;=\; \sigma\!\bigl(r(a)-r(b)\bigr),
\]
where $\sigma(x)=\frac{1}{1+e^{-x}}$ is the logistic sigmoid, and $r(\cdot)$ is the scalar ``reward'' function to be learned.  In maximum-likelihood training, we minimize the following negative log-likelihood (or equivalently cross-entropy) loss:
{\scriptsize
\[
\mathcal{L}(r)
\;=\;
\sum_{\substack{(a,b) \\ \text{pairs}}}\Bigl[
-\,p(a,b)\,\log\!\bigl(\sigma(r(a)-r(b))\bigr)
-\;p(b,a)\,\log\!\bigl(\sigma(r(b)-r(a))\bigr)
\Bigr].
\]
}
Here, $p(a,b)$ is the fraction of annotators that pick $a$ over $b$, so $p(a,b)+p(b,a)=1$.

\medskip
\textbf{2.\; Derivatives and stationarity.}
In the infinite-data limit, at the optimum $r^*$, the partial derivative $\frac{\partial\mathcal{L}}{\partial r(a)}$ must be zero for each $a\in\mathcal{A}$. We compute these derivatives carefully. Consider one pair $(a,b)$. Its contribution to $\mathcal{L}(r)$ is
\begin{align}
\ell_{a,b}(r) \;=\;
&-\,p(a,b)\,\log\!\bigl(\sigma(r(a)-r(b))\bigr) \nonumber \\
&-\;p(b,a)\,\log\!\bigl(\sigma(r(b)-r(a))\bigr).
\end{align}

Recall $\sigma(x)=1/(1+e^{-x})$, and $\sigma'(x) = \sigma(x)\bigl(1-\sigma(x)\bigr)$. We need:
\[
\frac{\partial}{\partial r(a)}\,\ell_{a,b}(r).
\]

\begin{itemize}
\item 
\underline{Term 1}, for $p(a,b)\log\sigma(r(a)-r(b))$:

\[\frac{\partial}{\partial r(a)}\Bigl[-\,p(a,b)\,\log\sigma(r(a)-r(b))\Bigr]\]
\[= -\,p(a,b)\,\bigl[\,1-\sigma(r(a)-r(b))\bigr].\]

\item 
\underline{Term 2}, for $p(b,a)\log\sigma(r(b)-r(a))$:

\[
\frac{\partial}{\partial r(a)}
\Bigl[-\,p(b,a)\,\log\sigma\bigl(r(b)-r(a)\bigr)\Bigr]\]
\[=\;
+\,p(b,a)\,\sigma\bigl(r(a)-r(b)\bigr).
\]
\end{itemize}

Hence, for a single pair $(a,b)$, its net derivative w.r.t.\ $r(a)$ is
\[
-\,p(a,b)\,\bigl[\,1-\sigma(r(a)-r(b))\bigr]
\;+\;
p(b,a)\,\sigma\bigl(r(a)-r(b)\bigr)
\]
\[
=-\,p(a,b)\,+\,p(a,b)\sigma(r(a)-r(b))
\;+\;
p(b,a)\,\sigma\bigl(r(a)-r(b)\bigr)
\]
\[
=\sigma\bigl(r(a)-r(b)\bigr)-\,p(a,b)\,
\]
Summing this over all pairs $\{(a,b):(a,b)\in\mathcal{A}\times\mathcal{A}\}$ that include $a$, we obtain
\[
\frac{\partial \mathcal{L}(r)}{\partial r(a)}
\;=\;
\sum_{b\neq a}\;\Bigl[
\sigma\bigl(r(a)-r(b)\bigr)-\,p(a,b)\,
\Bigr].
\]
At an optimum $r^*$, we require this derivative to be zero for every $a$:
\[
\sum_{b\neq a}\;\Bigl[
\sigma\bigl(r^*(a)-r^*(b)\bigr)-\,p(a,b)
\Bigr]
=
0.
\]
Hence the stationarity condition $\frac{\partial\mathcal{L}}{\partial r(a)}=0$ becomes:
\[
\sum_{b\neq a}\;\Bigl[
\sigma(\Delta_{ab})-\,p(a,b)
\Bigr]
=
0.
\]
where $\Delta_{ab} = r^*(a)-r^*(b)$.  Define $\sigma(\Delta_{ab})$ as $s_{ab}$ for shorthand. 
Therefore stationarity is
\[
\sum_{b\neq a}
\Bigl(s_{ab} \;-\; p(a,b)\Bigr)
\;=\;0
\;\;\;\Longleftrightarrow\;\;\;
\sum_{b\neq a} s_{ab}
\;=\;\sum_{b\neq a} p(a,b).
\]
Because $s_{ab}=\sigma\bigl(r^*(a)-r^*(b)\bigr)$ is monotonic in $r^*(a)-r^*(b)$, it follows that the items $\{a_1,\dots,a_m\}$ are ranked by $r^*(\cdot)$ exactly in ascending (or descending) order of $\sum_{b\neq a} p(a,b)$.  That is,
\[
r^*(a) > r^*(b)
\;\;\Longleftrightarrow\;\;
\sum_{c\neq a}\,p(a,c)
>
\sum_{c\neq b}\,p(b,c).
\]
\medskip
\textbf{3.\; Conclusion: equivalence with Borda ordering.}
Since $\text{BordaCount}(a)=\sum_{c\neq a} p(a,c)$, the learned function $r^*$ orders items $\{a\}$ in precisely the same way that their Borda counts do.  Thus
\[
r^*(a) > r^*(b)
\;\;\Longleftrightarrow\;\;
\text{BordaCount}(a)
>
\text{BordaCount}(b).
\]
This completes the proof.
\end{proof}

\subsection{Previous NLHF experiments}
\label{appendix:previous_NLHF_experiments}

A major advantage of our connection to NLHF is that existing experiments already compared Maximal Lotteries with RLHF with real datasets. Here, we summarize key findings from two studies that use the dataset of human-annotated summaries from \cite{stiennon2020learning}. This dataset consists of human-written summaries paired with preference labels, making it a realistic benchmark for LLM fine-tuning. Both papers use a T5X - Large LLM \cite{roberts2023scaling} as a starting point for their policies and a PaLM 2 \cite{anil2023palm} as an LLM judge. 

\textbf{Summarization Experiments in Nash-MD (NLHF)} The Nash-MD paper \cite{munos2023nash} shows that preference models, which directly predict the probability of one summary being preferred over another, achieve higher agreement with human ratings than reward models. Additionally, they show that Nash-MD outperforms their RLHF baseline in summarization. 

\textbf{Online IPO Experiments} The Online IPO study \cite{calandriello2024human} compares multiple algorithms, including RLHF - DPO \cite{rafailov2024direct}, Online IPO, and Nash-MD-PG. Their results confirm that Online IPO yield better-aligned models than RLHF, as measured by preferences from the judge.

\end{document}